\documentclass[11pt]{article}
\usepackage{url,ifthen}
\usepackage{srcltx}
\usepackage{multirow}
\usepackage{boxedminipage}
\usepackage[margin=1.1in]{geometry}
\usepackage{nicefrac}
\usepackage{xspace}
\usepackage{graphicx}
\usepackage{srcltx}
\usepackage[usenames]{color}
\usepackage{fullpage}
\usepackage{xspace}
\definecolor{DarkGreen}{rgb}{0.1,0.5,0.1}
\definecolor{DarkRed}{rgb}{0.5,0.1,0.1}
\definecolor{DarkBlue}{rgb}{0.1,0.1,0.5}
\usepackage[small]{caption}
\usepackage{float}
\usepackage{balance}
\usepackage{amsmath}
\usepackage{amsfonts}
\usepackage{amssymb}
\usepackage{amsthm}
\usepackage{bbm}
\usepackage{dsfont}

\newtheorem{theorem}{Theorem}[section]
\newtheorem*{namedtheorem}{\theoremname}
\newcommand{\theoremname}{testing}

\newtheorem{lemma}[theorem]{Lemma}

\newtheorem{claim}[theorem]{Claim}

\newtheorem{observation}{Observation}

\newtheorem*{question*}{Question}

\theoremstyle{definition}
\newtheorem{definition}[theorem]{Definition}
\newtheorem{defn}[theorem]{Definition}
\newtheorem{remark}[theorem]{Remark}

\theoremstyle{plain}
\newtheorem{Alg}{Algorithm}

\definecolor{DarkGreen}{rgb}{0.1,0.5,0.1}
\definecolor{DarkRed}{rgb}{0.5,0.1,0.1}
\definecolor{DarkBlue}{rgb}{0.1,0.1,0.5}

\usepackage[pdftex]{hyperref}
\hypersetup{
    unicode=false,          % non-Latin characters in Acrobat¿s bookmarks
    pdftoolbar=true,        % show Acrobat toolbar?
    pdfmenubar=true,        % show Acrobat menu?
    pdffitwindow=false,      % page fit to window when opened
    pdfnewwindow=true,      % links in new window
    colorlinks=true,       % false: boxed links; true: colored links
    linkcolor=DarkGreen,          % color of internal links
    citecolor=DarkGreen,        % color of links to bibliography
    filecolor=DarkGreen,      % color of file links
    urlcolor=DarkBlue,          % color of external links
}

\newcommand{\ignore}[1]{}

\renewcommand{\Pr}{\mathop{\bf Pr\/}}                    % should we change these to \mathbb for consistency of single-letter functionals
\newcommand{\E}{\mathop{\bf E\/}}

\newcommand{\Var}{\mathop{\bf Var\/}}
\newcommand{\Cov}{\mathop{\bf Cov\/}}

\newcommand{\empirical}{\widehat}

\makeatletter
\floatstyle{ruled}
\newfloat{fragment}{H}{lop}
\floatname{fragment}{Game}
\renewcommand{\floatc@ruled}[2]{\vspace{2pt}{\@fs@cfont \#1.\:} \#2 \par
 \vspace{1pt}}
\makeatother

\title{Information Theoretic Properties of Markov Random Fields, \\ and their Algorithmic Applications}
\author{Linus Hamilton \thanks{Massachusetts Institute of Technology. Department of Mathematics. Email: {\tt luh@mit.edu}.}
\and
Frederic Koehler \thanks{Massachusetts Institute of Technology. Department of Mathematics. Email: {\tt fkoehler@mit.edu}.}
\and
Ankur Moitra \thanks{
Massachusetts Institute of Technology. Department of Mathematics and the Computer Science and Artificial Intelligence Lab. Email: {\tt moitra@mit.edu}.
his work was supported in part by NSF CAREER Award CCF-1453261, NSF Large CCF-1565235, a David and Lucile Packard Fellowship and an Alfred P. Sloan Fellowship.}}

\begin{document}
\maketitle

\begin{abstract}
Markov random fields area popular model for high-dimensional probability distributions. Over the years, many mathematical, statistical and algorithmic problems on them have been studied. Until recently, the only known algorithms for provably learning them relied on exhaustive search, correlation decay or various incoherence assumptions. Bresler \cite{Bresler} gave an algorithm for learning general Ising models on bounded degree graphs. His approach was based on a structural result about mutual information in Ising models. 

Here we take a more conceptual approach to proving lower bounds on the mutual information through setting up an appropriate zero-sum game. Our proof generalizes well beyond Ising models, to arbitrary Markov random fields with higher order interactions. As an application, we obtain algorithms for learning Markov random fields on bounded degree graphs on $n$ nodes with $r$-order interactions in $n^r$ time and $\log n$ sample complexity. The sample complexity is information theoretically optimal up to the dependence on the maximum degree. The running time is nearly optimal under standard conjectures about the hardness of learning parity with noise. 
\end{abstract}

\thispagestyle{empty}

\newpage

\setcounter{page}{1}

\section{Introduction}

\subsection{Background}

{\em Markov random fields} are a popular model for defining high-dimensional distributions by using a graph to encode conditional dependencies among a collection of random variables. More precisely, the distribution is described by an undirected graph $G = (V, E)$ where to each of the $n$ nodes $u \in V$ we associate a random variable $X_u$ which takes on one of $k_u$ different states. The crucial property is that  the conditional distribution of $X_u$ should only depend on the states of $u$'s neighbors. It turns out that as long as every configuration has positive probability, the distribution can be written as
\begin{equation*}
\Pr(a_1, \cdots a_n) = \exp\left(\sum_{\ell = 1}^r \sum_{i_1 < i_2 < \cdots < i_{\ell} } \theta^{i_1 \cdots i_{\ell}}(a_1, \cdots a_n) - C\right)
\end{equation*}
Here $\theta^{i_1 \cdots i_{\ell}}: [k_{i_1}] \times \ldots \times [k_{i_\ell}] \rightarrow \mathbb{R}$ is a function that takes as input the configuration of states on the nodes $i_1, i_2, \cdots i_\ell$ and is assumed to be zero on non-cliques. These functions are referred to as {\em clique potentials}. In the equation above, $C$ is a constant  that ensures the distribution is normalized and is called the log-partition function. Such distributions are also called {\em Gibbs measures} and arise frequently in statistical physics and have numerous applications in computer vision, computational biology, social networks and signal processing. The {\em Ising model} corresponds to the special case where every node has two possible states and the only non-zero clique potentials correspond to single nodes or to pairs of nodes.

Over the years, many sorts of mathematical, statistical and algorithmic problems have been studied on Markov random fields. Such models first arose in the context of statistical physics where they were used to model systems of interacting particles and predict temperatures at which phase transitions occur \cite{history}. A rich body of work in mathematical physics aims to rigorously understand such phenomena. It is also natural to seek algorithms for sampling from the Gibbs distribution when given its clique potentials. There is a natural Markov chain to do so, and a number of works have identified a critical temperature (in our model this is a part of the clique potentials) above which the Markov chain mixes rapidly and below which it mixes slowly \cite{Martin, MosselWW}. Remarkably in some cases these critical temperatures also demarcate where approximate sampling goes from being easy to being computationally hard \cite{Sly, SlySun}. Finally, various inference problems on Markov random fields lead to graph partitioning problems such as the metric labelling problem \cite{KleinbergT}. 

In this paper, we will be primarily concerned with the {\em structure learning problem}. Given samples from a Markov random field, our goal is to learn the underlying graph $G$ with high probability. The problem of structure learning was initiated by Chow and Liu \cite{ChowL} who gave an algorithm for learning Markov random fields whose underlying graph is a tree by computing the maximum-weight spanning tree where the weight of each edge is equal to the mutual information of the variables at its endpoints. The running time and sample complexity are on the order of $n^2$ and $\log n$ respectively. Since then, a number of works have sought algorithms for more general families of Markov random fields. There have been generalizations to polytrees \cite{Dasgupta}, hypertrees \cite{Srebro} and tree mixtures \cite{treemixtures}. Others works construct the neighborhood by exhaustive search \cite{Ab, Csi, BreslerMosselSly}, impose certain incoherence conditions \cite{Lee, RavikumarWL, Jal} or require that there are no long range correlations (e.g. between nodes at large distance in the underlying graph) \cite{Anandkumar, BreslerMosselSly}. 

In a breakthrough work, Bresler \cite{Bresler} gave a simple greedy algorithm that provably works for any bounded degree Ising model, even if it has long-range correlations. This work used mutual information as its underlying progress measure and for each node it constructed its neighborhood. For a set $S$ of nodes, let $X_S$ denote the random variable representing their joint state. Then the key fact is the following:

\begin{quote}
For every node $u$, for any set $S \subseteq V \setminus \{u\}$  that does not contain all of $u$'s neighbors, there is a node $v \neq u$ which has non-negligible conditional mutual information (conditioned on $X_S$) with $u$.
\end{quote}

\noindent This fact is simultaneously surprising and not surprising. When $S$ contains all the neighbors of $u$, then $X_u$ has zero conditional mutual information (again conditioned on $X_S$) with any other node because $X_u$ only depends on $X_S$. Conversely shouldn't we expect that if $S$ does not contain the entire neighborhood of $u$, that there is some neighbor that has nonzero conditional mutual information with $u$? The difficulty is that the influence of a neighbor on $u$ can be cancelled out indirectly by the other neighbors of $u$. The key fact above tells us that it is impossible for the influences to all cancel out. But is this fact only true for Ising models or is it an instance of a more general phenomenon that holds over any Markov random field? 

\subsection{Our Techniques}

In this work, we give a vast generalization of Bresler's \cite{Bresler} lower bound on the conditional mutual information. We prove that it holds in general Markov random fields with higher order interactions provided that we look at sets of nodes. More precisely we prove, in a Markov random field with non-binary states and order up to $r$ interactions, the following fundamental fact:

\begin{quote}
For every node $u$, for any set $S \subseteq V \setminus \{u\}$that does not contain all of $u$'s neighbors, there is a set $I$ of at most $r-1$ nodes which does not contain $u$ where $X_u$ and $X_I$ have non-negligible conditional mutual information (conditioned on $X_S$).
\end{quote}

\begin{remark} It is necessary to allow $I$ to be a set of size $r-1$. For any integer $r$, there are Markov random fields where for every node $u$ and every set of nodes $I$ of size at most $r-2$, the mutual information between $X_u$ and $X_I$ is zero. 
\end{remark} %Thus our lower bound on the conditional mutual information is tight in the size of the sets $I$ that we consider. We defer a precise statement of our lower bound and its algorithmic implications to the next subsection. 

The starting point of our proof is a more conceptual approach to lower bounding the mutual information. Let $\Gamma(u)$ denote the neighbors of $u$. What makes proving such a lower bound challenging is that even though $I(X_u ; X_{\Gamma(u)}) > 0$, this alone is not enough to conclude that $I(X_u ; X_j) > 0$ for some $j \in \Gamma(u)$. Indeed for general distributions we can have that the mutual information between a variable and a set of variables is positive, but every pair of variables has zero mutual information. The distribution produced by a general Markov random field can be quite unwieldy (e.g. it is computationally hard to sample from) so what we need is a technique to tame it to make sure these types of pathologies cannot arise. 

Our approach goes through a two-player game that we call the {\sc GuessingGame} between Alice and Bob. Alice samples a configuration $X_1, X_2, \ldots X_n$ and reveals $I$ and $X_I$ for a randomly chosen set of $u$'s neighbors with $|I| \leq r-1$. Bob's goal is to guess $X_u$ with non-trivial advantage over its marginal distribution. We give an explicit strategy for Bob that achieves positive expected value. Our approach is quite general because we base Bob's guess on the contribution of $X_I$ to the overall clique potentials that $X_u$ participates in, in a way that the expectation over $I$ yields an unbiased estimator of the total clique potential. The fact that the strategy has positive expected value is then immediate, and all that remains is to prove a quantitative lower bound on it using the law of total variance. From here, the intuition is that if the mutual information $I(X_u ; X_{I})$ were zero for all sets $I$ then Bob could not have positive expected value in the {\sc GuessingGame}. This can be made precise and yields a lower bound on the mutual information. We can extend the argument to work with conditional mutual information by exploiting the fact that there are many clique potentials that do not get cancelled out when conditioning.

\subsection{Our Results}

Recall that $\Gamma(u)$ denotes the neighbors of $u$. We require certain conditions (Definition~\ref{def:degenerate}) on the clique potentials to hold, which we call $\alpha, \beta$-non-degeneracy, to ensure that the presence or absence of each hyperedge can be information-theoretically determined from few samples (essentially that no clique potential is too large and no non-zero clique potential is too small). Under this condition, we prove:

\begin{theorem}\label{mic-boundinf}
Fix any node $u$ in an $\alpha, \beta$-non-degenerate Markov random field of bounded degree
  and a subset of the vertices $S$ which does not contain the entire
  neighborhood of $u$. Then taking $I$ uniformly at random from the subsets of the neighbors of $u$
not contained in $S$ of size $s = \min(r-1,|\Gamma(u) \setminus S|)$, %and $R$ and $G$ uniformly at random from $[k_u]$, $[k_i]$
we have $\E_I[I(X_u;X_I | X_S)]  \ge C$.
\end{theorem}

\noindent See Theorem~\ref{mic-bound} which gives the precise dependence of $C$ on all of the constants, including $\alpha$, $\beta$, the maximum degree $D$, the order of the interactions $r$ and the upper bound $K$ on the number of states of each node. We remark that $C$ is exponentially small in $D$, $r$ and $\beta$ and there are examples where this dependence is necessary \cite{SanthanamW}. 

Next we apply our structural result within Bresler's \cite{Bresler} greedy framework for structure learning to obtain our main algorithmic result.  We obtain an algorithm for learning Markov random fields on bounded degree graphs with a logarithmic number of samples, which is information-theoretically optimal \cite{SanthanamW}. More precisely we prove:

\begin{theorem}\label{main-alginf}
Fix any $\alpha, \beta$-non-degenerate Markov random field on $n$ nodes with $r$-order interactions and bounded degree. There is an algorithm for learning $G$ that succeeds with high probability given $C' \log n$ samples and runs in time polynomial in $n^r$. 
\end{theorem}

\begin{remark}
An $r-1$-sparse parity with noise is a Markov random field with order $r$ interactions. This means if we could improve the running time to $n^{o(r)}$ this would yield the first $n^{o(k)}$ algorithm for learning $k$-sparse parities with noise, which is a long-standing open question. The best known algorithm of Valiant \cite{Valiant} runs in time $n^{0.8k}$. 
\end{remark}

\noindent See Theorem~\ref{theorem:main-result-formal} for a more precise statement. The constant $C'$ depends doubly exponentially on $D$. In the special case of Ising models with no external field, Vuffray et al. \cite{RISE} gave an algorithm based on convex programming that reduces the dependence on $D$ to singly exponential. In greedy approaches based on mutual information like the one we consider here, doubly-exponential dependence on $D$ seems intrinsic.  As in Bresler's \cite{Bresler} work, we construct a superset of the neighborhood that contains roughly $1/C$ nodes where $C$ comes from Theorem~\ref{mic-boundinf}. Recall that $C$ is exponentially small in $D$. Then to accurately estimate conditional mutual information when conditioning on the states of this many nodes, we need doubly exponential in $D$ many samples. 

However, there is a distinct advantage to greedy based methods. Since we only ever need to estimate the conditional mutual information on a constant sized sets of nodes and when conditioning on a constant sized set of other nodes, we can perform structure learning with partial observations. More precisely, if for every sample from a Markov random field, we are allowed to specify a set $J$ of size at most a constant $C''$ where all we observe is $X_J$ we can still learn the structure of the Markov random field. We call such queries $C''$-bounded queries. 

\begin{theorem}\label{main-alg2inf}
Fix any $\alpha, \beta$-non-degenerate Markov random field on $n$ nodes with $r$-order interactions and bounded degree. There is an algorithm for learning $G$ with $C''$-bounded queries that succeeds with high probability given $C' \log n$ samples and runs in time polynomial in $n^r$. 
\end{theorem}

\noindent See Theorem~\ref{theorem:main-theorem-bounded-queries} for a more precise statement. This natural scenario arises when it is too expensive to obtain a sample where the states of all nodes are known. The only other results we are aware of for learning with bounded queries work only for Gaussian graphical models \cite{Dasarathy}. We also consider a model where the state of each node is erased (i.e. we observe a `?' instead of its state) independently with some fixed probability $p$. See Theorem~\ref{theorem:main-result-random-erasure} for a precise statement. {\em The fact that we can straightforwardly obtain algorithms for these alternative settings demonstrates the flexibility of greedy, information-theoretic approaches to learning. }

In concurrent and independent work and using a different approach, Klivans and Meka \cite{KlivansM} gave an algorithm for learning Markov random fields with $r$-order interactions and maximum degree $D$ with a non-degeneracy assumption corresponding to a bound on the $\ell_1$-norm of the derivatives of the clique potentials.  Under our non-degeneracy assumptions, their algorithm runs in time $n^r$ and has sample complexity $ 2^{D^r} ( nr  )^r $ and under related but stronger assumptions than we use here, their sample complexity improves to $ 2^{D^r} r^r \log n$.

\section{Preliminaries}

For reference, all fundamental parameters of the graphical model
(max degree, etc.) are defined in the next two
subsections. In terms of these fundamental parameters, we define
additional parameters $\gamma$ and $\delta$ in (\ref{gammadelta}), $C(\gamma, K, \alpha)$ and $C'(\gamma,K,\alpha)$ in Theorem \ref{mi-bound}
and Theorem \ref{mic-bound} respectively, and $\tau$
in (\ref{tau}) and $L$ in (\ref{L}). 

\subsection{Markov Random Fields and the Canonical Form}

We study the problem of structure learning for Markov random fields. Formally, a Markov random field is specified by a hypergraph $\mathcal{H} = (V, H)$ where each hyperedge $h \in H$ is a set of at most $r$ vertices and $V = [n]$. To each node $i$, we associate a random variable $X_i$ which can take on one of $k_i$ different states/spins/colors so that $X_i \in [k_i]$. Let $K$ be an upper bound on the maximum number of states of any node. To each hyperedge $h = (i_1, i_2, \cdots i_\ell)$ we associate an $\ell$-order tensor $\theta^{i_1 i_2 \cdots i_{\ell}}$ with dimensions $k_{i_1} \times \cdots k_{i_{\ell}}$ which represents the clique interaction on these nodes. When $(i_1, i_2, \cdots i_\ell)$ are not a hyperedge in $\mathcal{H}$ we define $\theta^{i_1 i_2 \cdots i_{\ell}}$ to be the zero tensor.

The joint probability of the model being in state $X = (X_1, \ldots, X_n)$ is given by
\begin{equation}\label{pmf}
\Pr(X = x) = \exp\left(\sum_{\ell = 1}^r \sum_{i_1 < i_2 < \cdots < i_{\ell}} \theta^{i_1 \cdots i_{\ell}}(x_{i_1}, \ldots, x_{i_{\ell}}) - C\right)
\end{equation}
where $C$ is a constant chosen so that the sum over all configurations $x$ of $\Pr(x)$ is one. Equivalently, $Z = \exp(C)$ is the partition function. For notational
convenience, even when $i_1, \ldots, i_{\ell}$ are not sorted in increasing order, we
define
$\theta^{i_1 \cdots i_{\ell}}(a_1, \ldots, a_{\ell}) = \theta^{i'_1
  \cdots i'_{\ell}}(a'_1, \ldots, a'_{\ell})$ where the
$i'_1, \ldots, i'_{\ell}$ are the sorted version of
$i_1, \ldots, i_{\ell}$ and the $a'_1, \ldots, a'_{\ell}$ are the
corresponding copies of $a_1, \ldots, a_{\ell}$.

The parameterization above is not unique. It will be helpful to put it in a normal form as below. A \emph{tensor fiber} is the vector given by fixing all of the indices of the tensor except for one; this generalizes the notion of
row/column in matrices. For example for any $1 \le m \le \ell$, $i_1 < \ldots < i_m < \ldots i_\ell$ and $a_1, \ldots, a_{m -1}, a_{m + 1}, \ldots a_{\ell}$ fixed, the corresponding tensor fiber is the set of elements $\theta^{i_1 \cdots i_{\ell}}(a_1, \ldots, a_m, \ldots, a_{\ell})$ where $a_m$ ranges from $1$ to $k_{i_m}$. 

\begin{definition}
We say that the weights $\theta$ are in
\emph{canonical form}\footnote{This is the same as writing the log of the probability mass function according to the \emph{Efron-Stein decomposition} with respect to the uniform measure on colors; this decomposition is known to be unique. See e.g. Chapter 8 of \cite{boolean-functions}}
  if for every tensor $\theta^{i_1 \cdots i_{\ell}}$, the sum over all of the
\emph{tensor fibers} of $\theta^{i_1 \cdots i_{\ell}}$ is zero.
\end{definition}

Moreover we say that a tensor with the property that the sum over all tensor fibers is zero is a \emph{centered tensor}.  Hence having
a Markov random field in canonical form just means that all of the tensors corresponding to its clique potentials are centered. Next we prove that every Markov random field can be put in canonical form:

\begin{claim}
Every Markov random field can be put in canonical form
\end{claim}

\begin{proof}
We will recenter the tensors one by one without changing the law in \eqref{pmf}. 
Starting with an arbitrary parameterization, observe that if 
the sum along some tensor fiber is $s \ne 0$, we can subtract $s/k_{i_m}$ from each of the entries in
the tensor fiber, so the sum over the tensor fiber is now zero,  and add $s$ to $\theta^{i_{\sim m}}(a_{\sim m})$ 
without changing the law of $X$ in \eqref{pmf}. Here $i_{\sim m}$ is our notation for $i_1, \ldots, i_{m - 1}, i_{m + 1}, \ldots i_{\ell}$.  
By iterating this process from the tensors representing
the highest-order interactions down to the tensors representing the lowest-order interactions\footnote{We treat
$C$ as the lowest order interaction, so when we are subtracting from the $1$-tensors (vectors) $\theta^i$ to recenter them,
we add the corresponding amount to $C$.}, we 
obtain the desired canonical form. 
\end{proof}

\subsection{Non-Degeneracy}

We let $G = (V, E)$ be the graph we obtain from $\mathcal{H}$ by replacing every hyperedge with a clique. Let $d_i$ denote the degree of $i$ in $G$ and let $D$ be a bound on the maximum degree. Let $\Gamma(i)$ denote the neighborhood of $i$. Then as usual $G$ encodes the independence properties of the Markov random field. Our goal is to recover the structure of $G$ with high probability. In order to accomplish this, we will need to ensure that edges and hyperedges are non-degenerate. 

\begin{definition}
We say that a hyperedge $h$ is maximal if no other hyperedge of strictly larger size contains $h$. 
\end{definition}

Informally, we will require that every edge in $G$ is contained in some non-zero hyperedge, that all maximal hyperedges have at least one parameter bounded away from zero and that no entries are too large. More formally:

\begin{definition}\label{def:degenerate}
We say that a Markov random field is $\alpha$,$\beta$-non-degenerate if 
\begin{itemize}

\item[(a)] Every edge $(i, j)$ in the graph $G$ is contained in some hyperedge $h \in H$ where the corresponding tensor is non-zero.

\item[(b)] Every maximal hyperedge $h \in H$ has at least one entry lower
bounded by $\alpha$ in absolute value.

\item[(c)] Every entry of $\theta^{i_1 i_2 \cdots i_{\ell}}$ is
upper bounded by a constant $\beta$ in absolute value.

\end{itemize}
\end{definition}

\noindent We will refer to a hyperedge $h$ with an entry lower bounded by $\alpha$ in absolute value as $\alpha$-\emph{nonvanishing}. Each of the above non-degeneracy conditions is imposed in order to make learning $G$ information-theoretically possible. If an edge $(i, j)$ were not contained in any hyperedge with a non-zero tensor then we could remove the edge and not change the law in \eqref{pmf}. If a hyperedge contains only entries that are arbitrarily close to zero, we cannot hope to learn that it is there. We require $\alpha$-nonvanishing just for maximal hyperedges so that it is still possible to learn $G$. Finally if we did not have an upper bound on the absolute value of the entries, the probabilities in \eqref{pmf} could become arbitrarily skewed and there could be nodes $i$ where $X_i$ only ever takes on a single value.

\subsection{Bounds on Conditional Probabilities}\label{sec:bounds}

First we review properties of the conditional probabilities in a Markov random field as well as introduce some convenient notation which we will use later on. Fix a node $u$ and its neighborhood $U = \Gamma(u)$. Then for any  $R \in [k_u]$ we have
\begin{equation}
 P(X_u = R | X_U) = \frac{\exp(\mathcal{E}_{u,R}^X)}{\sum_{B = 1}^{k_u} \exp(\mathcal{E}_{u,B}^X)} \label{local-probability}
\end{equation}
where we define
\[ \mathcal{E}_{u,R}^X = \sum_{\ell = 1}^r \sum_{i_2 < \cdots < i_{\ell}} \theta^{u i_2 \cdots i_{\ell}}(R, X_{i_2}, \cdots, X_{i_\ell}) \]
and $i_2, \ldots, i_{\ell}$ range over elements of the neighborhood $U$; when $\ell = 1$ the inner sum is just $\theta^u(R)$.
To see that the above is true, first condition on $X_{\sim u}$, and see that the probability for a certain $X_u$ is proportional to $\exp(\mathcal{E}_{u,R}^X)$, which gives the right hand side of \eqref{local-probability}. Then apply the tower property for conditional probabilities.

Therefore if we define (where $|T|_{max}$ denotes the maximum entry of a tensor $T$)
\begin{equation}\label{gammadelta}
  \gamma := \sup_{u} \sum_{\ell = 1}^r \sum_{i_2 < \cdots < i_{\ell}} |\theta^{u i_2 \cdots i_{\ell}}|_{max} \le \beta \sum_{\ell =1}^{r}{D \choose \ell - 1}, \qquad \delta := \frac{1}{K} \exp(-2\gamma)
\end{equation}
then for any $R$
\begin{equation}
 P(X_u = R | X_U ) \ge \frac{\exp(-\gamma)}{K \exp(\gamma)} = \frac{1}{K} \exp(-2\gamma) = \delta
\end{equation}
Observe that if we pick any node
$i$ and consider the new Markov random field given by conditioning on a fixed
value of $X_i$, then the value of $\gamma$ for the new Markov random field is non-increasing. 

\subsection{Lower Bounds for Conditional Mutual Information}

As in Bresler's work on learning Ising models \cite{Bresler}, certain information theoretic quantities will play a crucial role as a progress measure in our algorithms. Specifically, we will use the functional
\[ \nu_{u,I| S} := \E_{R,G}\Big [\E_{X_S}\Big [\Big |\Pr(X_u = R,X_I = G | X_S) - \Pr(X_u = R | X_S)\Pr(X_I = G | X_S)\Big |\Big ]\Big ] \]
where $R$ is a state drawn uniformly at random from $[k_u]$, uniformly at random and $G$ is an $|I|$-tuple
of states drawn independently uniformly at random from $[k_{i_1}] \times [k_{i_2}] \times \ldots \times [k_{i_{|I|}}]$ where $I = (i_1, i_2, \ldots i_{|I|})$. 
This will be used as a {\em proxy} for conditional mutual information which can be efficiently
estimated from samples. The following lemma is a version of Lemma $5.1$ in \cite{Bresler} that works over non-binary alphabets. 
\begin{lemma}\label{nu-lemma}
Fix a set of nodes $S$. Fix a node $u$ and a set of nodes $I$ that are not contained in $S$. Then
\[ \sqrt{\frac{1}{2}I(X_u;X_I | X_S)} \ge \nu_{u,I | S} \]
\end{lemma}
\begin{proof}
\begin{align*}
\sqrt{\frac{1}{2}I(X_u;X_I | X_S)} &= \sqrt{\frac{1}{2}\E_{X_S = x_S}[I(X_u;X_I | X_S = x_S)]} \\
&\ge \E_{X_S = x_S}\left[\sqrt{\frac{1}{2}I(X_u;X_I | X_S = x_S)}\right] \\
&= \E_{X_S = x_S}\left[\sqrt{\frac{1}{2} D_{KL}(\Pr(X_u,X_I | X_S = x_S) || \Pr(X_u | X_S = x_S)\Pr(X_I | X_S = x_S))}\right] \\
&\ge \E_{X_S}[\sup_{R,G}[|\Pr(X_u = R, X_I = G | X_S) - \Pr(X_u = R | X_S)\Pr(X_I = G | X_S)|]] \\
&\ge \E_{X_S}[\E_{R,G}[|\Pr(X_u = R, X_I = G | X_S) - \Pr(X_u = R | X_S)\Pr(X_I = G | X_S)|]] \\
&= \nu_{u,I |S}.
\end{align*}
where the first inequality follows from Jensen's inequality, and the second inequality follows from Pinsker's inequality.
\end{proof}

\subsection{No Cancellation}

In this subsection we will show that a clique
interaction of order $s$ cannot be completely cancelled out by clique
interactions of lower order. 
% NOTE: There may be an alternative bound which depends on K.
% in fact, by using parseval for Efron-Stein we can get a uniform bound
% like \mu \prod k_i \le \mu K^s on the entries of all tensors, which is not
% necessarily better, but we can also get an even better
% bound if we are not looking at infinity norms...
% POSSIBLE TODO: redefine alpha-nonvanishing in l2 sense and modify our arguments?
\begin{lemma}\label{lemma:noncancellation-support}
  Let $T^{1 \cdots s}$ be a centered tensor of dimensions $d_1 \times \cdots \times d_s$
  and suppose there exists at least one entry of $T^{1 \cdots s}$ which is lower bounded in absolute
  value by a constant $\kappa$. For any $\ell < s$ and $i_1 < \cdots < i_{\ell}$ let $T^{i_1 \cdots i_{\ell}}$
  be an arbitrary centered tensor of dimensions $d_{i_1} \times \cdots \times d_{i_{\ell}}$. Define
  \begin{equation}\label{effective-tensor}
    T(a_1, \ldots, a_s) = \sum_{\ell = 1}^s \sum_{i_1 < \cdots < i_{\ell}} T^{i_1 \cdots i_{\ell}}(a_{i_1}, \ldots, a_{i_{\ell}})
  \end{equation}
  and suppose the entries of $T$ are bounded by a constant $\mu$.
  Then for any $\ell$ and $i_1 < \cdots < i_{\ell}$, the entries
  of $T^{i_1 \cdots i_{\ell}}(a_{i_1}, \ldots, a_{i_{\ell}})$ are bounded above
  by $\mu \ell^{\ell}$.
\end{lemma}
\begin{proof}
The sum over all values of indices $a_1, \ldots, a_s$ on the right hand side is zero, so the same must hold for the left hand side. Assume for
contradiction that every entry of $T$ is upper bounded by $\mu$, to be optimized later. For each $m$ from $1$ to $s$, consider
summing over all of the indices except $a_m$, which is held fixed. Using that the sum over tensor fibers is zero, we 
observe that the right hand side of \eqref{effective-tensor} is just
\[ T^{i_m}(a_1) \prod_{m' \ne m} d_{m'} \]
and the left hand side is strictly bounded in norm by $\mu \prod_{m' \ne m} d_{m'} $
so $|T^{i_m}(a_m)| < \mu $
for all $a_m$. We have proven this for all $m$ from $1$ to $s$.

Now we proceed by induction, assuming that $t$ indices are fixed. We will show that the entries of the $t$-tensors are bounded above by $\mu g(t)$
for $g(t) = 2^{t(t + 1)/2}$ and have already proven this for $t =1$. Now suppose we fix $a_1, \ldots, a_t$. We rearrange \eqref{effective-tensor} to get
\begin{align*}
  &T(a_1, \ldots, a_s) - \sum_{\ell = 1}^{t - 1} \sum_{\{i_1 < \cdots < i_{\ell}\} \subset [t]} T^{i_1 \cdots i_{\ell}}(a_{i_1}, \ldots, a_{i_{\ell}}) \\
  &\quad = T^{1 \cdots t}(a_1, \ldots, a_t) + \sum_{\ell = 1}^s\sum_{\{i_1 < \cdots < i_{\ell}\}\not\subset [t]} T^{i_1 \cdots i_{\ell}}(a_{i_1}, \ldots, a_{i_{\ell}})
\end{align*}
When we fix indices  $a_1, \ldots, a_t$ and sum over the others, all but the first term on the rhs vanishes, and by applying the triangle inequality on the lhs and the induction hypothesis we get that
\[ d_{t + 1} \cdots d_{s} \left(\mu + \sum_{\ell = 1}^{t - 1} {t \choose \ell} \mu g(\ell)\right) > d_{t + 1} \cdots d_{s} T^{u i_2 \cdots i_t}(a_1, \ldots, a_t) \]
so taking $g(t)$ such that $g(0) = 1$ and
\[ g(t) \ge \sum_{\ell = 0}^{t - 1} {t \choose \ell} g(\ell) \]
and in particular $g(t) = t^t$ works, because
\[ t^t = (1 + (t - 1))^t = \sum_{\ell = 0}^{t} {t \choose \ell} (t - 1)^{\ell} \ge \sum_{\ell = 0}^{t - 1} {t \choose \ell} \ell^{\ell}. \]
%FIXME REWRITE ABOVE. ALSO $g(t) = O(t!) = O(e^{t \log t})$ WORKS.
% https://oeis.org/search?q=1%2C2%2C9%2C52%2C375%2C3246&sort=&language=english&go=Search
Thus we get that all the entries of $T^{i_1 \cdots i_{\ell}}(a_{i_1}, \ldots, a_{i_{\ell}})$ are bounded above
  by $\mu \ell^{\ell}$, which completes the proof. 
\end{proof}

We are now ready to restate the above result in a more usable form:

\begin{lemma}\label{lemma:noncancellation}
  Let $T^{1 \cdots s}$ be a centered tensor of dimensions $d_1 \times \cdots \times d_s$
  and suppose there exists at least one entry of $T^{1 \cdots s}$ which is lower bounded in absolute
  value by a constant $\kappa$. For any $\ell < s$ and $i_1 < \cdots < i_{\ell}$ let $T^{i_1 \cdots i_{\ell}}$
  be an arbitrary centered tensor of dimensions $d_{i_1} \times \cdots \times d_{i_{\ell}}$. Let
\begin{equation}\label{effective-tensor}
T(a_1, \ldots, a_s) = \sum_{\ell = 1}^s \sum_{i_1 < \cdots < i_{\ell}} T^{i_1 \cdots i_{\ell}}(a_{i_1}, \ldots, a_{i_{\ell}})
\end{equation}
Then the sum over all the entries of $T$ is 0, and there exists an
entry of $T$ of absolute value lower bounded by $\kappa/s^s$.
\end{lemma}
\begin{proof}
We apply the previous lemma with $\mu = \kappa/s^s$, and  get that all the entries of $T^{1 \cdots s}$ are bounded in
absolute value by $\mu s^s$, giving a contradiction.
\end{proof}

\section{The Guessing Game}

Here we introduce a game-theoretic framework for understanding mutual information in general Markov random fields. The {\sc GuessingGame} is defined as follows:

\begin{center}
\noindent\rule{15cm}{0.4pt}
\end{center}

\begin{enumerate} \itemsep 0pt
\small 
\item Alice samples $X = (X_1, \ldots, X_n)$ and $X' = (X'_1, \ldots, X'_n)$ independently from the Markov random field
\item Alice samples $R$ uniformly at random from $[k_u]$
\item Alice samples a set $I$ of size $s = \min(r-1, d_u)$ uniformly at random from the neighbors of $u$
\item Alice tells Bob $I$, $X_I$ and $R$
\item Bob wagers $w$ with $|w| \leq \gamma K {D \choose r - 1}$
\item Bob gets $\Delta = w \mathds{1}_{X_u = R} - w \mathds{1}_{X'_u = R}$
\end{enumerate} 

\begin{center}
\noindent\rule{15cm}{0.4pt}
\end{center}

Bob's goal is to guess $X_u$ given knowledge of the states of some of $u$'s neighbors. The Markov random field (including all of its parameters) are common knowledge. The intuition is that if Bob can obtain a positive expected value, then there must be some set $I$ of neighbors of $u$ which have non-zero mutual information. In this section, will show that there is a simple, explicit strategy for Bob that yields positive expected value. 

\subsection{A Good Strategy for Bob}

Here we will show an explicit strategy for Bob that has positive expected value. Our analysis will rest on the following key lemma:

\begin{lemma}\label{lemma:bobstrat}
There is a strategy for Bob that wagers at most $\gamma K {D \choose r - 1}$ in absolute value that satisfies $$\E_{I, X_I} [w | X_{\sim u}, R] = \mathcal{E}_{u, R}^X - \sum_{B \neq R} \mathcal{E}_{u, B}^X$$
\end{lemma}

\begin{proof}
First we explicitly define Bob's strategy. Let
$$\Phi(R, I, X_I) = \sum_{\ell = 1}^{s} C_{u, \ell, s} \sum_{ i_1 < i_2 < \cdots < i_{\ell}} \mathds{1}_{\{i_1 \cdots i_\ell \} \subseteq I} \theta^{u i_1 \cdots i_{\ell}}(R, X_{i_1}, \ldots, X_{i_{\ell}})$$
where $C_{u, \ell, s} = \frac{{d_u \choose s}}{ {d_u - \ell \choose s - \ell}}$. Then Bob wagers
$$ w = \Phi(R, I, X_I) - \sum_{B \neq R} \Phi(B, I, X_I) $$
Notice that the strategy only depends on $X_I$ because all terms in the summation where $\{i_1 \cdots i_\ell\}$ are not a subset of $I$ have zero contribution. 

The intuition behind this strategy is that the weighting term satisifes $$C_{u, \ell, s} = \frac{1}{\Pr[\{i_1,\ldots i_\ell\} \subset I]}$$ Thus when we take the expectation over $I$ and $X_I$ we get
$$\E_{I, X_I} [\Phi(R, I, X_I) | X_{\sim u}, R] = \sum_{\ell = 1}^r \sum_{i_2 < \cdots < i_{\ell}} \theta^{u i_2 \cdots i_{\ell}}(R, X_{i_2}, \cdots, X_{i_\ell}) = \mathcal{E}_{u, R}^X $$
and hence $\E_{I, X_I}[w | X_{\sim u}, R] = \mathcal{E}_{u, R}^X - \sum_{B \neq R} \mathcal{E}_{u, B}^X$. To complete the proof, notice that $C_{u, \ell, s} \leq {D \choose r-1}$ which using the definition of $\gamma$ implies that $|\Phi(R, I, X_I)| \leq \gamma {D \choose r-1}$ for any state $B$, and thus Bob wagers at most the desired amount (in absolute value). 
\end{proof}

Now we are ready to analyze the strategy:

\begin{theorem}\label{thm:maingame}
There is a  strategy for Bob that wagers at most $\gamma K {D \choose r - 1}$ in absolute value which satisfies $$\E [\Delta] \geq \frac{4\alpha^2 \delta^{r-1}}{r^{2r} e^{2\gamma}}$$
\end{theorem}

\begin{proof}
We will use the strategy from Lemma~\ref{lemma:bobstrat}. First we fix $X_{\sim u}$, $X'_{\sim u}$ and $R$. Then we have
$$\E_{I, X_I}[\Delta  | X_{\sim u}, X'_{\sim u}, R] = \E_{I, X_I}[w | X_{\sim u}, R] \Big ( \Pr[X_u=R|X_{\sim u}, R]-\Pr[X'_u=R|X'_{\sim u}, R]\Big ) $$
which follows because $\Delta = r \mathds{1}_{X_u = R} - r \mathds{1}_{X'_u = R}$ and because $r$ and $X_u$ do not depend on $X'_{\sim u}$  and similarly $X'_u$ does not depend on $X_{\sim u}$ . Now using (\ref{local-probability}) we calculate:
\begin{align*}
\Pr[X_u=R|X_{\sim u}, R]-\Pr[X'_u=R|X'_{\sim u}, R] &=  \frac{\exp(\mathcal{E}_{u,R}^X)}{\sum_B \exp(\mathcal{E}_{u,B}^X)} -  \frac{\exp(\mathcal{E}_{u,R}^{X'})}{\sum_B \exp(\mathcal{E}_{u,B}^{X'})} \\
&= \frac{1}{D} \Big ( \sum_{B \neq R} \exp(\mathcal{E}_{u,R}^X + \mathcal{E}_{u,B}^{X'}) - \exp(\mathcal{E}_{u,B}^X + \mathcal{E}_{u,R}^{X'}) \Big )
\end{align*}
where $D = \Big (\sum_B \exp(\mathcal{E}_{u,B}^X)\Big ) \Big(\sum_B \exp(\mathcal{E}_{u,B}^{X'})\Big ) $.
Thus putting it all together we have
$$\E_{I, X_I}[\Delta | X_{\sim u}, X'_{\sim u}, R] = \frac{1}{D} \Big ( \mathcal{E}_{u, R}^X - \sum_{B \neq R} \mathcal{E}_{u, B}^X \Big ) \Big ( \sum_{B \neq R} \exp(\mathcal{E}_{u,R}^X + \mathcal{E}_{u,B}^{X'}) - \exp(\mathcal{E}_{u,B}^X + \mathcal{E}_{u,R}^{X'}) \Big )$$
Now it is easy to see that 
$$\sum_{\mbox{distinct } R, G, B} \mathcal{E}_{u,B}^X \left( \sum_{G \neq R} \exp(\mathcal{E}_{u,R}^X+ \mathcal{E}_{u,G}^{X'}) - \exp(\mathcal{E}_{u,G}^X + \mathcal{E}_{u,R}^{X'}) \right) = 0 $$
which follows because when we interchange $R$ and $G$ the entire term multiplies by a negative one and so we can pair up the terms in the summation so that they exactly cancel. Using this identity we get
$$\E_{I, X_I}[\Delta | X_{\sim u}, X'_{\sim u}] = \frac{1}{k_u D} \sum_R \sum_{B \neq R} \Big ( \mathcal{E}_{u, R}^X -  \mathcal{E}_{u, B}^X \Big ) \Big (  \exp(\mathcal{E}_{u,R}^X + \mathcal{E}_{u,B}^{X'}) - \exp(\mathcal{E}_{u,B}^X + \mathcal{E}_{u,R}^{X'}) \Big )$$
where we have also used the fact that $R$ is uniform on $k_u$. And finally using the fact that $X_{\sim u}$ and $X'_{\sim u}$ are identically distributed we can sample $Y_{\sim u}$ and $Z_{\sim u}$ and flip a coin to decide whether we set $X_{\sim u} = Y_{\sim u}$ and $X'_{\sim u} = Z_{\sim u}$ or vice-versa. Now we have
$$\E_{I, X_I}[\Delta | Y_{\sim u}, Z_{\sim u}] = \frac{1}{2k_u D} \sum_R  \sum_{B \neq R} \Big ( \mathcal{E}_{u, R}^Y -  \mathcal{E}_{u, B}^Y -  \mathcal{E}_{u, R}^Z + \mathcal{E}_{u, B}^Z \Big ) \Big (  \exp(\mathcal{E}_{u,R}^Y + \mathcal{E}_{u,B}^{Z}) - \exp(\mathcal{E}_{u,B}^Y + \mathcal{E}_{u,R}^{Z}) \Big )$$
With the appropriate notation it is easy to see that the above sum is strictly positive. Let $a_{R, B} = \mathcal{E}_{u, R}^Y +  \mathcal{E}_{u, B}^Z$ and $b_{R, B} = \mathcal{E}_{u, R}^Z + \mathcal{E}_{u, B}^Y$.  With this notation:
 $$ \E_{I, X_I}[\Delta | Y_{\sim u}, Z_{\sim u}] = \frac{1}{2 D k_u } \sum_R \sum_{B \neq R} \Big ( a_{R, B} - b_{R, B} \Big ) \Big ( \exp(a_{R, B}) -  \exp(b_{R, B}) \Big )$$
Since $\exp(x)$ is a strictly increasing function it follows that as long as $a_{R, B} \neq b_{R, B}$ for some term in the sum, the sum is positive. In Lemma~\ref{lemma:quantbound} we prove that the expectation over $Y$ and $Z$ of this sum is at least $\frac{4\alpha^2 \delta^{r-1}}{r^{2r} e^{2\gamma}}$, which completes the proof.
\end{proof}

\subsection{A Quantitative Lower Bound}

Here we prove a quantitative lower bound on the sum that arose in the proof of Theorem~\ref{thm:maingame}. More precisely we show:

\begin{lemma}\label{lemma:quantbound}
$$ \E_{Y, Z} \Big [ \sum_R \sum_{B \neq R} \Big ( \mathcal{E}_{u, R}^Y -  \mathcal{E}_{u, B}^Y -  \mathcal{E}_{u, R}^Z + \mathcal{E}_{u, B}^Z \Big ) \Big (  \exp(\mathcal{E}_{u,R}^Y + \mathcal{E}_{u,B}^{Z}) - \exp(\mathcal{E}_{u,B}^Y + \mathcal{E}_{u,R}^{Z}) \Big ) \Big ]\ge \frac{4\alpha^2 \delta^{r-1}}{r^{2r} e^{2\gamma}}$$
\end{lemma}
\begin{proof}
Setting $a =\mathcal{E}_{u, R}^Y +  \mathcal{E}_{u, B}^Z $ and $b = \mathcal{E}_{u, B}^Y + \mathcal{E}_{u, R}^Z$, letting $D' = K^3 \exp(2 \gamma) \ge D$, and
 taking an expectation over the randomness in $Y$ and $Z$, we have
\begin{align*}
\E_{Y, Z} \Big [ \sum_R \sum_{R \ne B} (a - b)(e^a -e^b) \Big ]
& =  \E[\sum_R \sum_{R \ne B} (a - b)\int_b^a e^x dx] \\
&\geq  \E[\sum_R \sum_{R \ne B} (a - b)^2 e^{-2\gamma}] \geq \frac{1}{ e^{2\gamma}} \sum_R \sum_{R \ne B} \Var[a - b]
\end{align*}
where the inequality follows from the fact that $a, b \geq -2 \gamma$. In the following claim, we give a more convenient expression for the above quantity.

\begin{claim}
$$\sum_R \sum_{R \ne B} \Var[a - b] = 4 k_u \sum_R \Var[\mathcal{E}_{u, R}^Y]$$
\end{claim}

\begin{proof}
Using the fact that $a - b = (\mathcal{E}_{u, R}^Y - \mathcal{E}_{u, B}^Y) + (\mathcal{E}_{u, B}^Z - \mathcal{E}_{u, R}^Z )$ we have that
\begin{align*}
%&\frac{1}{2D' e^{2\gamma}} \sum_{R \ne B} \Var[a - b] \\
%&\frac{1}{2D' e^{2\gamma}} \sum_{R \ne B} \Var[(\Eve{X}{R} + \Eve{Y}{B}) - (\Eve{Y}{R} + \Eve{X}{B})] \\
\sum_R \sum_{R \ne B} \Var[a - b] &= \sum_R \sum_{B \neq R} \Var[(\mathcal{E}_{u, R}^Y - \mathcal{E}_{u, B}^Y) + (\mathcal{E}_{u, B}^Z - \mathcal{E}_{u, R}^Z )]  \\
&= 2 \sum_R \sum_{B \neq R} \Var[(\mathcal{E}_{u, R}^Y - \mathcal{E}_{u, B}^Y)] \\
%&= 2 \sum_R \sum_{B \neq R} \Big ( \Var[\mathcal{E}_{u, R}^Y] - 2 \Cov \Big ( \mathcal{E}_{u, R}^Y,\mathcal{E}_{u, B}^Y  \Big )  + \Var[\mathcal{E}_{u, B}^Y] \Big ) \\
&= 2 \sum_R \sum_{B \neq R} \Big ( 2 \Var[\mathcal{E}_{u, R}^Y] - 2 \Cov \Big ( \mathcal{E}_{u, R}^Y,\mathcal{E}_{u, B}^Y  \Big ) \Big ) \\
&= 2 \sum_R  \Big ( 2 (k_u - 1) \Var[\mathcal{E}_{u, R}^Y] - 2 \Cov \Big ( \mathcal{E}_{u, R}^Y, \sum_{B \neq R}\mathcal{E}_{u, B}^Y  \Big ) \Big ) \\
&= 2 \sum_R  \Big ( 2 (k_u - 1) \Var[\mathcal{E}_{u, R}^Y] - 2 \Cov \Big ( \mathcal{E}_{u, R}^Y, -\mathcal{E}_{u, R}^Y  \Big ) \Big ) \\
&= 4 k_u \sum_R \Var[\mathcal{E}_{u, R}^Y]
\end{align*}
where the second to last equality follows from the fact that the tensors are centered which gives $\sum_R \mathcal{E}_{u, R}^Y = 0$ for any $Y$. This completes the proof. 
\end{proof}

Now we can complete the proof by appealing to the law of total variance. By assumption there is a maximal hyperedge $J = \{u, j_1 \ldots j_s\}$ containing $u$ with $|J| \le r$, such that $\theta^{uJ}$ is $\alpha$-nonvanishing. Then we have
$$ \sum_{R} \Var[\mathcal{E}_{u, R}^Y] \ge \sum_{R} \Var[\mathcal{E}_{u, R}^Y | Y_{\sim J}] = \sum_{R} \Var[T(R, Y_{j_1}, \ldots, Y_{j_s}) | Y_{\sim J}]$$
where the tensor $T$ is defined by treating $Y_{\sim J}$ as fixed as follows
$$T(R,Y_{j_1}, \ldots, Y_{j_s}) =\sum_{\ell = 2}^r \sum_{i_2 < \cdots < i_{\ell}} \theta^{u i_2 \cdots i_{\ell}}(R, Y_{i_2}, \cdots, Y_{i_{\ell}})$$
Now we claim there is a choice of $R$, $G$ and $G'$ so that $|T(R,G) - T(R,G')| > \alpha/r^r$. This follows because from Lemma~\ref{lemma:noncancellation} we have that  $T$ is
$\alpha/r^r$-nonvanishing. Hence there is a choice of $R$ and $G$ so that $|T(R,G)| > \alpha/r^r$. Because $T$ is centered there must be a $G'$ so that $T(R, G')$ has the opposite sign. 

Finally for this choice of $R$ we have
$$ \Var[T(R, Y_{j_1}, \ldots, Y_{j_s}) | Y_{\sim J}] \geq \frac{\alpha^2 \delta^{r-1}}{2r^{2r}}$$
which follows from the fact that $\Pr(Y_{J \setminus u} = G)$ and $\Pr(Y_{J \setminus u} = G')$ are both lower bounded by $\delta^{r - 1}$ and the following elementary lower bound on the variance:

\begin{claim}
Let $Z$ be a random variable such that $Pr(Z = a) \ge p$ and $Pr(Z = b) \ge p$, then
\[ \Var(Z) \ge \frac{p}{2}(a - b)^2 \]
\end{claim}
\begin{proof}
$$\Var(Z) \ge p(a - \E[Z])^2 + p(\E[Z] - b)^2 \ge p\Big(a - \frac{a + b}{2}\Big)^2 + p\Big(b - \frac{a + b}{2}\Big)^2 = \frac{p}{2}(a - b)^2$$
\end{proof}

Putting this all together we have
$$\E_{Y, Z} \Big [ \sum_R \sum_{R \ne B} (a - b)(e^a -e^b) \Big ] \geq \frac{4\alpha^2 \delta^{r-1}}{r^{2r} e^{2\gamma}}$$
which is the desired bound. This completes the proof.
\end{proof}

\section{Implications for Mutual Information}

In this section we show that Bob's strategy implies a lower bound on the mutual information between node $u$ and a subset $I$ of its neighbors of size at most $r-1$. We then extend the argument to work with conditional mutual information as well. 

\subsection{Mutual Information in Markov Random Fields}\label{sec:mutual}

Recall that the goal of the {\sc GuessingGame} is for Bob to use information about the states of nodes $I$ to guess the state of node $u$. Intuitively, if $X_I$ conveys no information about $X_u$ then it should contradict the fact that Bob has a strategy with positive expected value. We make this precise below. Our argument proceeds in two steps. First we upper bound the expected value of any strategy. 

\begin{lemma}\label{lemma:gametoinf1}
For any strategy, $$\E[\Delta] \leq \gamma K {D \choose r - 1} \E_{I, X_I, R}\Big[|\Pr[X_u = R | X_I] - \Pr[X_u = R]|\Big]$$
\end{lemma}

\begin{proof}
Intuitively this follows because Bob's optimal strategy given $I$, $X_I$ and $R$ is to guess $$w = \mbox{sgn}(\Pr[X_u = R | X_I] - \Pr[X_u = R]) \gamma K$$ More precisely, we have 
\begin{align*}
\E[\Delta] &= \E_{I, X_I, R}\Big [\E_{X_{\sim I}, X'} \Big[ r \mathds{1}_{X_u = R} - r \mathds{1}_{X'_u = R} \Big | I, X_I, R\Big ] \Big] \\
&= \E_{I, X_I, R}\Big [ r \Pr[X_u = R | X_I] - r \Pr[X'_u = R] \Big] \\
&= \E_{I, X_I, R}\Big [ r \Pr[X_u = R | X_I] - r \Pr[X_u = R] \Big] \\
& \leq \gamma K {D \choose r - 1} \E_{I, X_I, R}\Big [|\Pr[X_u = R | X_I] - \Pr[X_u = R]|\Big]
\end{align*}
which completes the proof. 
\end{proof}

Next we lower bound the mutual information using (essentially) the same quantity. We prove

\begin{lemma}\label{lemma:gametoinf2}
$$\sqrt{\frac{1}{2} I(X_u;X_I)} \geq \frac{1}{K^r} \E_{X_I, R}\Big[ |\Pr(X_u = R | X_I) - \Pr(X_u = R)| \Big] $$
\end{lemma}

\begin{proof}
Applying Lemma~\ref{nu-lemma} with $S = \emptyset$ we have that
\begin{align*}
\sqrt{\frac{1}{2} I(X_u;X_I)}
& \ge \E_{R,G}\Big[|\Pr(X_u = R, X_I = G) - \Pr(X_u = R)\Pr(X_I = G)|\Big] \\
&= \E_{R,G}\Big[\Pr(X_I = G) |\Pr(X_u = R | X_I = G) - \Pr(X_u = R)|\Big] \\
&= \frac{1}{\prod_{i \in I} k_i} \sum_G \Pr(X_I = G) \E_{R}[|\Pr(X_u = R | X_I = G) - \Pr(X_u = R)|] \\
& \ge \frac{1}{K^r} \E_{R,X_I}\Big[ |\Pr(X_u = R | X_I) - \Pr(X_u = R)|\Big]
\end{align*}
where $R$ and $G$ are uniform (as in the definition of $\nu_{u,I| S}$).
\end{proof}

Now appealing to Lemma~\ref{lemma:gametoinf1}, Lemma~\ref{lemma:gametoinf2} and Theorem~\ref{thm:maingame} we conclude:

\begin{theorem}\label{mi-bound}
  Fix a non-isolated vertex $u$ contained in at least one
  $\alpha$-nonvanishing maximal hyperedge.
  Then taking $I$ uniformly at random from the subsets
  of the neighbors of $u$ of size $s = \min(r-1,deg(u))$,
%and $R$ and $G$ uniformly at random from $[k_u], [k_i]$
\begin{align*}
\E_I\left[\sqrt{\frac{1}{2}I(X_u;X_I)}\right] 
%&\ge \E_{i,R,G}[|\Pr(X_u = R,X_i = G) - \Pr(X_u = R)\Pr(X_i = G)|] \\
\ge \E_I[\nu_{u,I | \emptyset}]
\ge C(\gamma,K,\alpha)
\end{align*}
where explicitly
\[ C(\gamma, K,\alpha) := \frac{4\alpha^2 \delta^{r-1}}{r^{2r}  K^{r+1} {D \choose r - 1} \gamma e^{2\gamma}} \]
\end{theorem}

\subsection{Extensions to Conditional Mutual Information}\label{sec:condmutual}

In the previous subsection, we showed that $X_u$ and $X_I$ have positive mutual information. Here we show that the argument extends to conditional mutual information when we condition on $X_S$ for any set $S$ that does not contain all the neighbors of $u$. The main idea is to show that there is a setting of $X_S$ where the hyperedges do not completely cancel out each other in the new Markov random field we obtain by conditioning on $X_S$. 

More precisely fix a set of nodes $S$ that does not contain all the neighbors of $u$ and let $I$ be chosen uniformly at random from the subsets of neighbors of $u$ of size $s = \min(r-1, |\Gamma(u) \setminus S|)$. Then we have
\begin{align*}
\E_I[\sqrt{\frac{1}{2}I(X_u;X_I | X_S)}] &= \E_I[\sqrt{\frac{1}{2}\E_{X_S = x_S}[I(X_u;X_I | X_S = x_S)]}] \\
&\ge \E_{I, X_S = x_S}\left[\sqrt{\frac{1}{2}I(X_u;X_I | X_S = x_S)}\right]
\end{align*}
which follows from Jensen's inequality. Now conditioned on $X_S = x_S$ the resulting distribution is again a Markov random field and $\gamma$ does not increase. 

\begin{definition}
Let $E$ be the event that conditioned on $X_S = x_S$, node $u$ is contained in at least one $\alpha/r^r$-nonvanishing maximal hyperedge. 
\end{definition}

\begin{lemma}\label{mic-helper}
$ \Pr(E) \ge \delta^{d} $
\end{lemma}

\begin{proof}
When we fix $X_S = x_S$ we obtain a new Markov random field where the underlying hypergraph is 
$$\mathcal{H}' := ([n] \setminus S, H') \mbox{ where } H' = \{h \setminus S | h \in H)$$ For notational convenience let $\phi(h)$ be the image of a hyperedge $h$ in $\mathcal{H}$ in the new hypergraph $\mathcal{H}'$. What makes things complicated is that a hyperedge in $\mathcal{H}'$ can have numerous preimages. The crux of our argument is in how to select the right one to show is $\alpha/r^r$-nonvanishing. First we observe that $u$ is contained in at least one non-empty hyperedge in $\mathcal{H}'$. This is because by assumption $S$ does not contain all the neighbors of $u$. Hence there is some neighbor $v \notin S$. Since $v$ is a neighbor of $u$ it means that there is a hyperedge $h \in H $ that contains both $u$ and $v$. In particular $\phi(h)$ contains $u$ and is nonempty. 

Now that we know $u$ is not isolated in $\mathcal{H}'$, let $h^*$ be a hyperedge in $\mathcal{H}$ that contains $u$ and where $\phi(h^*)$ is maximal.  Now let $f_1, f_2, \ldots f_p$ be the preimages of $\phi(h^*)$ so that without loss of generality $f_1$ is maximal in $\mathcal{H}$. Now let $J = \cup_{i=1}^p f_i \setminus \{u\}$. In particular, $J$ is the set of neighbors of $u$ that are contained in at least one of $f_1, f_2, \ldots f_p$. Finally let $J_1 = J \cap S := \{i_1, i_2, \ldots i_s\}$ and let $J_2 = J \setminus S := \{i'_1, i'_2, \ldots i'_{s'}\}$. We can now define
$$T(R, a_1, \ldots, a_s, a'_1, \ldots, a'_{s'}) = \sum_{i=1}^p \theta^{f_i}$$
which is the clique potential we get on hyperedge $\phi(h^*)$ when we fix each index in $J_1 \subseteq S$ to their corresponding value. 

Suppose for the purposes of contradiction that all the entries of $T$ are strictly bounded in absolute value by $\alpha/r^r$. Then applying Lemma~\ref{lemma:noncancellation-support} in the contrapositive we see that the entries of $f_1$ are strictly bounded above in absolute value by $\alpha$, but $f_1$  is maximal and thus $\alpha$-nonvanishing, which yields a contradiction. Thus there is some setting $a_1^*,
  \ldots, a_s^*$ such that the tensor
  $$ T'(R, a'_1, \ldots, a'_{s'}) = T(R, a_1^*, \ldots, a_s^*, a'_1, \ldots, a'_{s'}) $$
  has at least one entry with absolute value at least $\alpha/r^r$. Under this setting, $\phi(h^*)$ is $\alpha/r^r$-nonvanishing and by construction maximal in $\mathcal{H}'$ and thus we would be done. All that remains is to lower bound the probability of this setting. Since $J_1$ is a subset of the neighbors of $u$ we have $|J_1| \leq d$. Thus the probability that
  $(X_{i_1}, \ldots, X_{i_s}) = (a_1^*, \ldots, a_s^*)$ is bounded
  below by $\delta^s \ge \delta^d$, which completes the proof. 
\end{proof}

Now we are ready to prove a lower bound on conditional mutual information:

\begin{theorem}\label{mic-bound}
  Fix a vertex $u$ such that all of the maximal hyperedges containing
  $u$ are $\alpha$-nonvanishing, 
  and a subset of the vertices $S$ which does not contain the entire
  neighborhood of $u$. Then taking $I$ uniformly at random from the subsets of the neighbors of $u$
not contained in $S$ of size $s = \min(r-1,|\Gamma(u) \setminus S|)$, %and $R$ and $G$ uniformly at random from $[k_u]$, $[k_i]$
\begin{align*}
\E_I\left[\sqrt{\frac{1}{2}I(X_u;X_I | X_S)}\right]  \ge E_{I}[\nu_{u,I|S}] \ge C'(\gamma,K, \alpha)
\end{align*}
where explicitly
\[ C'(\gamma, K,\alpha) := \frac{4\alpha^2 \delta^{r + d - 1}}{r^{2r}  K^{r+1} {D \choose r - 1} \gamma e^{2\gamma}} \]
\end{theorem}

\begin{proof}
We have
$$\E_{I, X_S}\left[\sqrt{\frac{1}{2}I(X_u;X_I | X_S)}\right] \geq \E_{I, X_S = x_S }\left[\sqrt{\frac{1}{2}I(X_u;X_I | X_S = x_S)} \mathbbm{1}_E\right] \geq\delta^d C(\gamma,K,\alpha) $$
where the last inequality follows by invoking Lemma~\ref{mic-helper} and applying Theorem~\ref{mi-bound} to the new Markov random field we get by conditioning on $X_S = x_S$. 
\end{proof}

\section{Applications}\label{sec:app}
\subsection{Learning Markov Random Fields}
% Setting: we are allowed to choose to observe the state of $C$ nodes from
% each sample drawn from the graphical model. Result: learning with doubly-exponential
% in $d$ many samples, $C$ single-exponential in $d$. 
% related to active learning fixme
We now employ the greedy approach of Bresler \cite{Bresler} which was previously used to learn Ising models on bounded degree graphs. Let $x^{(1)}, \ldots, x^{(m)}$ denote a collection of independent samples from the underlying Markov random field. Let $\empirical\Pr$ denote the empirical distribution so that
\[ \empirical\Pr(X = x) = \frac{1}{m} \sum_{i = 1}^m \mathds{1}_{x^{(i)} = x}. \]
Let $\empirical\E$ denote the expectation under
this distribution, i.e. the sample average. 

In our algorithm, we will need estimates for $\nu_{u,i | S}$ which we obtain in the usual way
by replacing all expectations over $X$ with sample averages:
\[ \empirical\nu_{u,i | S} := \E_{R,G}\empirical{\E}_{X_S}[|\empirical{\Pr}(X_u = R,X_i = G | X_S) - \empirical{\Pr}(X_u = R | X_S)\empirical{\Pr}(X_i = G | X_S)|] \]
Also we define $\tau$ (which will be used as a thresholding constant) as
\begin{equation}\label{tau}
  \tau := C'(\gamma,k,\alpha)/2
\end{equation}
and $L$, which is an upper bound on the size of the superset of a neighborhood of $u$ that the algorithm will construct,
\begin{equation}\label{L}
  L := (8/\tau^2) \log K = (32/C'(\gamma,k,\alpha)^2) \log K.
\end{equation}
Then the algorithm {\sc MrfNbhd} at node $u$ is:
\begin{center}
\noindent\rule{15cm}{0.4pt}
\end{center}
\begin{enumerate}
\item Fix input vertex $u$. Set $S := \emptyset$.
\item While  $|S| \le L$ and there exists a set of vertices $I \subset [n] \setminus S$ of size at most $r - 1$
  such that $\empirical \nu_{u,I | S} > \tau$, set $S := S \cup I$.
\item For each $i \in S$, if $\empirical \nu_{u,i | S \setminus i} < \tau$ then remove $i$ from $S$.
\item Return set $S$ as our estimate of the neighborhood of $u$.
\end{enumerate}

\begin{center}
\noindent\rule{15cm}{0.4pt}
\end{center}

The algorithm will succeed provided that $\empirical\nu_{u,I|S}$ is sufficiently close to
the true value $\nu_{u,I|S}$. This motivates the definition of the event $A$:
\begin{defn}  
  We denote by $A(\ell,\epsilon)$ the event that for all $u$, $I$ and $S$ with $|I| \le r - 1$ and $|S| \le \ell$ simultaneously,
  \[ \left|\nu_{u,i | S} - \empirical\nu_{u,i | S}\right| < \epsilon. \]  
  We let $A$ denote the event $A(L, \tau/2)$. 
\end{defn}

The proofs of the following technical lemma is left to an appendix.
\begin{lemma}\label{lemma:nasty}
  Fix a set $S$ with $|S| \le \ell$ and suppose that for any set $T \supseteq S$ with
  $|T \setminus S| \le r$, that
  \[ |\empirical{\Pr}(X_T = x_T) - \Pr(X_T = x_T)| < \sigma. \]
  If $\sigma \le \epsilon K^{-\ell} \frac{\delta^{\ell}}{5}$ then for any $I$ with 
  $|I| \le r - 1$,
  \[ \left|\nu_{u,i | S} - \empirical\nu_{u,i | S}\right| < \epsilon. \]    
\end{lemma}
\begin{lemma}\label{lemma:mbound}
  Fix $\ell, \epsilon$ and $\omega > 0$. 
  If the number of samples satisfies
\[ m \ge \frac{15 K^{2\ell}}{\epsilon^2 \delta^{2 \ell}} \Big (\log(1/\omega) + \log(\ell + r) + (\ell + r)\log(nK) + \log 2 \Big ) \]
 then $\Pr(A(\ell,\epsilon)) \ge 1 - \omega$. 
\end{lemma}
\begin{proof}[Proof of Lemma~\ref{lemma:mbound}]
  Fix $\ell, \epsilon$ and $\omega > 0$. Let $m$ denote the number
  of samples. By Hoeffding's inequality, for any set $T$,
\[ \Pr[|\empirical{\Pr}(X_T = x_T) - \Pr(X_T = x_T)| > \sigma] \le 2\exp(-2
  \sigma^2 m) \] and taking the union bound over all possibly $x_T$
for $T$ with $|T| \le \ell + r$, of which there are at most

\[ \sum_{i = 1}^{\ell + r} {n \choose i} K^{i} \le \sum_{i = 1}^{\ell + r} (nK)^i \le (\ell + r)(nK)^{\ell + r} \]
many, we find the probability that $|\empirical{\Pr}(X_T = x_T) - \Pr(X_T = x_T)| > \sigma$ for any such $T$ is at most
\[ (\ell + r)(nK)^{\ell + r} 2\exp(-2 \sigma^2 m) \] Therefore taking
\begin{equation}\label{mbound1}
m \ge \frac{\log(1/\omega) + \log(\ell + r) + (\ell + r)\log(nK) +
  \log 2}{2 \sigma^2}
\end{equation} ensures this probability is at most
$\omega$.

Now applying Lemma~\ref{lemma:nasty} and substituting $\sigma = \epsilon K^{-\ell} \frac{\delta^{\ell}}{5}$ into \eqref{mbound1}, we see that the result holds if
\[ m \ge \frac{15 K^{2\ell}}{\epsilon^2 \delta^{2 \ell}} \Big (\log(1/\omega) + \log(\ell + r) + (\ell + r)\log(nK) + \log 2 \Big ) \]
\end{proof}

\global\long\def\e{\epsilon}
\global\long\def\EE#1#2{\mathbb{E}_{#1}\left[#2\right]}

\begin{lemma}\label{lemma:mi-increase}
Assume that the event $A$ holds. Then every time a node $i$
is added to $S$ in Step 2 of the algorithm,
the mutual information $I(X_{u};X_{S})$ increases
by at least $\tau^2/8$. \end{lemma}
\begin{proof}
For a particular iteration of Step 2, let $I$ denote the newly
added set of nodes, and $S$ the set of candidate neighbors before adding
$I$. Then we must show for $Q = \tau^2/8$ that
\[
I(X_{u};X_{S\cup\{I\}})\geq I(X_{u};X_{S})+Q
\]
which by the chain rule for expectation is equivalent to
\[
I(X_{u};X_{I}|X_{S})\geq Q.
\]

\iffalse
Let $Q(u_{R})$ denote $\Pr[X_{u}=R|X_{S}]$, and similarly $Q(u_{R}|i_{G}):=\Pr[X_{u}=R|X_{i}=G,X_{S}]$,
for any colors $R$ and $G$. Then

\begin{eqnarray*}
\sqrt{\frac{1}{2}\cdot I(X_{u};X_{i}|X_{S})} & = & \sqrt{\frac{1}{2}\sum_{x_{S}}\Pr[x_{S}]I(X_{u};X_{i}|X_{S}=x_{S})}\\
 & \geq & \sum_{x_{S}}P\left[x_{S}\right]\sqrt{\frac{1}{2}\cdot I(X_{u};X_{i}|X_{S}=x_{S})}\mbox{\hspace{1em}(Jensen's inequality, as the square root function is concave)}\\
 & = & \EE{X_{S}}{\sqrt{\frac{1}{2}\cdot D_{\mbox{KL}}(Q(u,i)||Q(u)Q(i))}}\\
 & \geq & \EE{X_{S}}{D_{\mbox{TV}}(Q(u,i),Q(u)Q(i))}\mbox{\hspace{1em}(Pinsker's inequality)}\\
 & \geq & \EE{X_{S}}{Q(u_{R},i_{G})-Q(u_{R})Q(i_{G})}\mbox{\hspace{1em}for any colors \ensuremath{R,G} (definition of total variation distance)}\\
 & = & \EE{X_{S}}{Q(i_{G})Q(\mbox{not }i_{G})\cdot\left|Q(u_{R}|i_{G})-Q(u_{R}|\mbox{not }i_{G})\right|}.
\end{eqnarray*}

The last line holds for all colors $R$ and $G$. Therefore

\fi
Applying Lemma~\ref{nu-lemma} and the fact that event $A$ holds, we see
\begin{eqnarray*}
\sqrt{\frac{1}{2}\cdot I(X_{u};X_{I}|X_{S})}
% \geq & \EE{R,G}{\EE{X_{S}}{Q(i_{G})Q(\mbox{not }i_{G})\cdot\left|Q(u_{R}|i_{G})-Q(u_{R}|\mbox{not }i_{G})\right|}}\\
 \ge \frac{1}{2}\nu_{u,I|S}
 \ge \frac{1}{2}\left(\empirical \nu_{u,i|S}-\tau/2\right)
\end{eqnarray*}
Thus the algorithm only adds node $i$ to $S$ if $\empirical \nu_{u,i|S} \geq \tau$,
so the chain of inequalities implies that $$I(X_{u};X_{i}|X_{S})\geq\frac{1}{2}(\tau - \tau/2)^{2} = \tau^2/8$$
\end{proof}
\begin{lemma}\label{lemma:greedy-neighborhood}
If event $A$ holds then at the end of Step 2, $S$ contains all of
the neighbors of $u$. 
\end{lemma}
\begin{proof}
Step 2 ended either because
  $|S| > L$ or because there was no set of nodes $I \subset [n] \setminus S$
with $\empirical\nu_{u,I | S} > \tau$. First we rule out the former
  possibility. Whenever a new element is added to $S$, the quantity
  $I(X_{u};X_{S})$ increases by at least $\tau^2/8$. But $$I(X_{u};X_{S}) \leq H(X_{u})\leq\log K$$ because
$X_u$ takes on at most $K$ states. Thus if $|S| > L$ then $$\log K \ge I(X_u;X_S) >
  L(\tau^2/8) = \log K$$ which gives a contradiction.

Thus at the end of Step 2 we must have that there is no set of nodes $I \subset [n] \setminus S$ with $\empirical\nu_{u,I | S} >
    \tau$. Suppose for the purposes of contradiction that $S$ does not contain all of the neighbors of
  $u$. Then by Theorem~\ref{mic-bound}, there exists a
  subset of the neighbors such that $\nu_{u,I |S} \ge
  C'(\gamma,k,\alpha) = 2\tau$, and because event $A$ holds we know
  $\empirical \nu_{u,I | S} > 2\tau - \tau/2 > \tau$, which gives us our contradiction and completes the proof of the lemma. 
\end{proof}
\begin{lemma}\label{lemma:prune-neighborhood}
If event $A$ holds and if at the start of Step 3 $S$ contains all neighbors of $u$, then at the end of Step 3 the remaining set of nodes are exactly the neighbors of $u$. 
\end{lemma}
\begin{proof}
If $A(\ell)$ holds, then during Step 3, 
\[ \empirical \nu_{u,i|S\backslash\{i\}}< \nu_{u,i | S} + \tau/2 \le \sqrt{\frac{1}{2} I(X_u;X_i | X_S)} + \tau/2 = \tau/2 \]
for all nodes $i$ that are not neighbors of $u$. Thus all such nodes are pruned. Furthermore, by Theorem~\ref{mic-bound}, $\empirical \nu_{u,i|S\backslash\{i\}}> \nu_{u,i | S\setminus \{i\}} - \tau/2 \ge 2\tau - \tau/2 = 3\tau/2$
for all neighbors of $u$ and thus no neighbor is pruned. This completes the proof. \end{proof}

\noindent Recall that $\gamma \leq \beta r D^r$, $\delta = e^{-2\gamma}/K$, $(C'(\gamma,K,\alpha))^{-1} = O(\frac{K^{r + 1} r^{2r}}{\alpha^2 
\delta^{2D}} D^{r - 1} \gamma e^{-2\gamma})$ and $L = O(C'(\gamma,K,\alpha)^{-2})$. 

\begin{theorem}%[Formal version of Theorem~\ref{theorem:main-result}]
\label{theorem:main-result-formal}
  Fix $\omega > 0$. Suppose we are given $m$ samples from an $\alpha, \beta$-non-degenerate Markov random field with $r$-order interactions where the underlying graph has maximum degree at most $D$ and each node takes on at most $K$ states. Suppose that
\[ m \ge  \frac{60 K^{2L}}{\tau^2 \delta^{2 L}} \Big (\log(1/\omega) + \log(L + r) + (L + r)\log(nK) + \log 2\Big ). \]
  Then with probability at least $1 - \omega$, 
  {\sc MrfNbhd} when run starting from each node $u$ recovers the correct neighborhood of $u$,
  and thus recovers the underlying graph $G$. Furthermore, each run of the algorithm takes $O(mLn^{r})$ time.
\end{theorem}
\begin{proof}
Set $\ell = L$ and $\epsilon = \tau/2$ in 
  Lemma~\ref{lemma:mbound}. Then event $A$ occurs with probability at least $1 - \omega$ for our choice of $m$.
Now by Lemma~\ref{lemma:greedy-neighborhood}
  and Lemma~\ref{lemma:prune-neighborhood} the algorithm returns the
  correct set of neighbors of $u$ for every node $u$.

  To analyze the running time, observe that when running algorithm
  \textsc{MrfNbhd} at a single node $u$, the bottleneck is Step 2, in which there are at most
  $L$ steps and in each step the algorithm must loop over
  all subsets of the vertices in $[n] \setminus S$ of size $r - 1$, of
  which there are $\sum_{\ell = 1}^{r - 1} {n \choose \ell} = O(n^{r - 1})$
  many. Running the algorithm at all nodes thus takes $O(mL n^r)$ time.
\end{proof}
\begin{remark} 
Note that when we plug in the values of $\gamma$ and $\delta$ we get that the overall sample complexity of our algorithm in terms of $D$ and $r$ is doubly exponential in $D^r$. %because $\gamma = O(\beta r D^r) = O(\beta D^{r + 1})$,
%$\delta = e^{-2\gamma}/K = \Omega(K^{-1} e^{-\beta D^{r + 1}})$,
%$(C'(\gamma,K,\alpha))^{-1} = O(\alpha^{-2} K^{r + 1} r^{2r}
%\delta^{-2D} \gamma e^{-2\gamma}) = O(\alpha^{-2} K^{3D + 2} r^{2r}
%e^{5\beta D^{r + 2}})$, and $R = O(C'(\gamma,K,\alpha)^{-2})$ the
%sample complexity in terms of $D$ and $r$ is of order $e^{e^{O(D^{r + 2})}}$.
\end{remark}

\subsection{Learning with Bounded Queries}

In many situations, it is too expensive to obtain full samples from a Markov random field (e.g. this could involve needing to measure every potential symptom of a patient). Here we consider a model where we are allowed only partial observations in the form of a $C$-bounded query:

\begin{defn}
 A \emph{$C$-bounded query} to a Markov random field
  is specified by a set $S$ with $|S| \le C$ and we observe $X_S$
\end{defn}

Our algorithm \textsc{MrfNbhd} can be made to work with $C$-bounded queries
instead of full observations by a simple change: instead of estimating all of the terms
$\empirical\nu_{u,I|S}$ jointly from samples, we estimate each
one individually by querying a fresh a batch of $m'$ samples of
$\{u\} \cup I \cup S$ every time the algorithm needs requires
$\empirical\nu_{u,I|S}$. First we make an elementary observation about \textsc{MrfNbhd}: 

\begin{observation}
In Step 2, \textsc{MrfNbhd} only needs $\empirical \nu_{u,I | S}$ for all $I$ with $|I| \le r - 1$.
  Similarly at Step 3, \textsc{MrfNbhd} only needs $\empirical \nu_{u, i | S \setminus i}$ for each $i \in S$.
\end{observation}

 \noindent Thus the number of distinct terms $\empirical \nu_{u,I |S}$ which
 \textsc{MrfNbhd} needs is at most $L (r - 1) n^{r - 1}$ for Step $2$
  and $R$ for Step $3$, which in total is at most $L r n^{r - 1}$. 

\begin{lemma}\label{lemma:mbound-bounded}
  Fix a node $u$, a set $S$ with $\ell = |S|$, a set $I$ with $|I| \le r - 1$ and fix
  $\epsilon$ and $\omega > 0$. 
  If the number of samples we observe of $X_{S \cup I \cup \{u\}}$ satisfies
  \[ m' \ge \frac{15 K^{2\ell}}{\epsilon^2 \delta^{2 \ell}} \Big (\log(1/\omega) + \log(\ell + r) + (\ell + r)\log(nK) + \log 2 \Big ) \]
  then
  \[ |\nu_{u,I | S} - \empirical\nu_{u,I |S }| < \epsilon \]
  with probability at least $1 - \omega$. 
\end{lemma}
\begin{proof}
  This follows by the same Hoeffding and union bound as in proof of
  Lemma~\ref{lemma:mbound}.
\end{proof}

\begin{theorem}\label{theorem:main-theorem-bounded-queries}
  Fix an $\alpha, \beta$-non-degenerate Markov random field with
  $r$-order interactions where the underlying graph has maximum degree
  at most $D$ and each node takes on at most $K$ states.  The bounded
  queries modification to the algorithm returns the correct
  neighborhood of every vertex $u$ using $m' L r n^{r}$-bounded
  queries of size at most $L+r$ where
\[ m' = \frac{60 K^{2L}}{\tau^2 \delta^{2 L}} \Big (\log(L r n^r/\omega) + \log(L + r) + (L + r)\log(nK) + \log 2 \Big ), \]    
  with probability at least $1-\omega$.
\end{theorem}

\begin{proof}
  Invoking Lemma~\ref{lemma:mbound-bounded} with $\omega' =
  \frac{\omega}{L r n^{r}}$, $\epsilon = \tau/2$ and $\ell = L$, we
  get that each query to $\empirical \nu_{u,I | S}$ fails (i.e. is
  wrong by at least $\tau/2$) with probability at most
  $\frac{\omega}{L r n^{r}}$. We observed that Algorithm
  \textsc{MrfNbhd} makes at most $L r n^{r-1}$ queries of the form,
  $\empirical \nu_{u,I | S}$. Therefore, by a union bound, with
  probability at least $1-\omega/n$, the bounded queries algorithm
  answers all of those queries to within tolerance $\tau/2$.
  
Now it follows as in Theorem~\ref{theorem:main-result-formal}
  that the algorithm returns the correct neighborhood of node
  $u$ with probability at least $1 - \omega/n$, and taking the union
  bound over all nodes $u$ it follows that the algorithm recovers
  the correct neighborhood of all nodes with probability at least
  $1 - \omega$. This completes the proof. 
\end{proof}

\subsection{Learning with Random Erasures}
Here we consider another variant where we do not observe full samples from a Markov random field. Instead we 
observe partial samples where the state of each node is revealed independently with probability $p$ and is otherwise replaced with a `?', and the choice of which nodes to reveal is independent
of the sample. We can apply our algorithm in this setting, as
follows.
\begin{lemma}\label{lemma:random-erasure-good}
With probability at least $1 - \epsilon$, if we take $N\frac{\ell \log n + \log \ell + \log N/\epsilon}{p^2}$ samples
then we will see each set $S$ at least $N$ times for every $|S| \le \ell$.
\end{lemma}
\begin{proof}
Each sample has at least a $p^{\ell}$ chance of being observed, and there
are at most $\ell n^{\ell}$ many different sets $S$. So by a union bound,
\[ Pr[\text{exists unobserved $S$ after $t$ steps}] \le n^{\ell}(1 - p^{\ell})^t \le \epsilon/N \]
if we take
$t = \frac{\ell \log n + \log \ell + \log N/\epsilon}{p^2}$.
Repeating this $N$ times, we see that with 
\[ Nt = N\frac{\ell \log n + \log \ell + \log N/\epsilon}{p^2} \]
many samples, we see every $S$ at least $N$ times with probability at least $1 - \epsilon$.
\end{proof}
\begin{lemma}\label{lemma:random-erasure-A}
  Fix $\ell, \epsilon$ and $\omega > 0$. 
  If the number of samples satisfies
  \[ m \ge N\frac{\ell \log N + \log \ell + \log 2N/\omega}{p^2} \]
  where
\[ N = \frac{15 K^{2\ell}}{\epsilon^2 \delta^{2 \ell}} \Big (\log(2/\omega) + \log(\ell + r) + (\ell + r)\log(nK) + \log 2 \Big)  \]
 then $\Pr(A(\ell,\epsilon)) \ge 1 - \omega$. 
\end{lemma}
\begin{proof}
  Observe by Lemma~\ref{lemma:random-erasure-good}, taking $\epsilon =
  \omega/2$ that with probability at least $1 - \omega/2$, for every
  set $S$ with $|S| \leq \ell$ we see at least $N$ samples revealing
  all of the members of $S$. Condition on this event; now the proof is
  exactly the same as Lemma~\ref{lemma:mbound} taking $\omega' =
  \omega/2$, applying Hoeffding, Lemma~\ref{lemma:nasty}
  and taking the union bound, we see
  that event $A$ holds with probability at least $\omega/2$. Therefore
  the total probability $A$ occurs is at least $1 - \omega/2 -
  \omega/2 = 1 - \omega$.
\end{proof}
\begin{theorem}%[Formal version of Theorem~\ref{theorem:main-result}]
\label{theorem:main-result-random-erasure}
Fix $\omega > 0$. Suppose we are given $m$ samples from an $\alpha, \beta$-non-degenerate Markov random field with $r$-order interactions where the underlying graph has maximum degree at most $D$ and each node takes on at most $K$ states. Suppose that
  \[ m \ge N\frac{\ell \log n + \log L + \log 2N/\omega}{p^2} \]
  where
  \[ N = \frac{60 K^{2L}}{\tau^2 \delta^{2 L}} \Big (\log(2/\omega) + \log(L + r) + (L + r)\log(nK) + \log 2\Big ). \]  
  Then with probability at least $1 - \omega$, 
  {\sc MrfNbhd} when run starting from each node $u$ recovers the correct neighborhood of $u$,
  and thus recovers the underlying graph $G$. Furthermore, each run of the algorithm takes $O(mLn^{r})$ time.
\end{theorem}
\begin{proof}
  By Lemma~\ref{lemma:random-erasure-A}, given our assumption on $m$
  the event $A$ occurs with probability at least $1 - \omega$. Conditioned
  on event $A$, the algorithm returns the correct answer by the same
  argument as Theorem~\ref{theorem:main-result-formal}.
\end{proof}

\newpage

\bibliographystyle{plain}
\bibliography{bibliography}

\newpage

\appendix

\section{Appendix: Proof of Lemma~\ref{lemma:nasty}}
\begin{proof}
  Observe the left hand side of our desired inequality is bounded by
\begin{align*}
E_{R,G}\big|&\empirical{\E}_{X_S}[|\empirical{\Pr}(X_u = R,X_I = G | X_S) - \empirical{\Pr}(X_u = R | X_S)\empirical{\Pr}(X_I = G | X_S)|] \\
&- \E_{X_S}[|\Pr(X_u = R,X_I = G | X_S) - \Pr(X_u = R | X_S)\Pr(X_I = G | X_S)|]\big|
\end{align*}
So it suffices if we can bound for every $R$ and $G$ 
\begin{align*}
\Big|&\empirical{\E}_{X_S}[|\empirical{\Pr}(X_u = R,X_I = G | X_S) - \empirical{\Pr}(X_u = R | X_S)\empirical{\Pr}(X_I = G | X_S)|] \\
&\qquad- \E_{X_S}[|\Pr(X_u = R,X_I = G | X_S) - \Pr(X_u = R | X_S)\Pr(X_I = G | X_S)|]\Big| \\
&=\Big|\sum_{x_S} |\empirical{\Pr}(X_u = R,X_I = G,X_S = x_S) - \empirical{\Pr}(X_u = R | X_S = x_S)\empirical{\Pr}(X_I = G, X_S = x_S)| \\
&\qquad- |\Pr(X_u = R,X_i = G, X_S = x_S) - \Pr(X_u = R | X_S = x_S)\Pr(X_I = G, X_S = x_S)|\Big| \\
&\le \sum_{x_S} \Big||\empirical{\Pr}(X_u = R,X_I = G,X_S = x_S) - \empirical{\Pr}(X_u = R | X_S = x_S)\empirical{\Pr}(X_I = G,X_S = x_S)| \\
&\qquad- |\Pr(X_u = R,X_I = G, X_S = x_S) - \Pr(X_u = R | X_S = x_S)\Pr(X_I = G,X_S = x_S)|\Big| \\
&\le \sum_{x_S} \Big|\empirical{\Pr}(X_u = R,X_I = G,X_S = x_S) - \empirical{\Pr}(X_u = R | X_S = x_S)\empirical{\Pr}(X_I = G,X_S = x_S) \\
&\qquad- \Pr(X_u = R,X_I = G,X_S = x_S) + \Pr(X_u = R | X_S = x_S)\Pr(X_I = G,X_S = x_S)\Big| \\
&\le \sum_{x_S} |\empirical{\Pr}(X_u = R,X_I = G,X_S = x_S) - \Pr(X_u = R,X_I = G,X_S = x_S)|  \\
&\qquad +\sum_{x_S}|\empirical{\Pr}(X_u = R | X_S = x_S)\empirical{\Pr}(X_I = G,X_S = x_S) - \Pr(X_u = R | X_S)\Pr(X_I = G,X_S = x_S)| \\
&\le K^{|S|} \sigma + \sum_{x_S} |\empirical{\Pr}(X_u = R | X_S = x_S)\empirical{\Pr}(X_I = G,X_S = x_S) - \Pr(X_u = R | X_S = x_S)\Pr(X_I = G,X_S = x_S)|
\end{align*}
To bound the second term, observe
\begin{align*}
|\empirical{\Pr}&(X_u = R | X_S = x_S)\empirical{\Pr}(X_I = G,X_S = x_S) - \Pr(X_u = R | X_S)\Pr(X_I = G,X_S = x_S)| \\
&\le \empirical{\Pr}(X_u = R | X_S = x_S)|\empirical{\Pr}(X_I = G,X_S = x_S) - \Pr(X_I = G, X_S = x_S)|  \\
&\qquad \qquad \qquad + \Pr(X_I = G,X_S = x_S)|\empirical{\Pr}(X_u = R | X_S = x_S) - \Pr(X_u = R | X_S)| \\
&\le |\empirical{\Pr}(X_I = G,X_S = x_S) - \Pr(X_I = G,X_S = x_S)| + |\empirical{\Pr}(X_u = R | X_S = x_S) - \Pr(X_u = R | X_S)| \\
&\le \sigma + |\empirical{\Pr}(X_u = R | X_S = x_S) - \Pr(X_u = R | X_S)|
\end{align*}
and furthermore
\begin{align*}
|\empirical{\Pr}&(X_u = R | X_S = x_S) - \Pr(X_u = R | X_S = x_S)| = \left|\frac{\empirical{\Pr}(X_u = R, X_S = x_S)}{\empirical{\Pr}(X_S = x_S)} - \frac{\Pr(X_u = R, X_S = x_S)}{\Pr(X_S = x_S)} \right| 
\end{align*}
which in turn we can bound as
\begin{align*}
&\le \left|\frac{\empirical{\Pr}(X_u = R, X_S = x_S)}{\empirical{\Pr}(X_S = x_S)} - \frac{\Pr(X_u = R, X_S = x_S)}{\empirical{\Pr}(X_S = x_S)}\right| \\
&\quad + \left|\frac{\Pr(X_u = R, X_S = x_S)}{\empirical{\Pr}(X_S = x_S)} - \frac{\Pr(X_u = R, X_S = x_S)}{\Pr(X_S = x_S)} \right| \\
&\le \frac{\sigma}{\delta^{|S|}} + \Pr(X_u = R, X_S = x_S)\left|\frac{\Pr(X_S = x_S) - \empirical{\Pr}(X_S = x_S)}{\empirical{\Pr}(X_S = x_S) \Pr(X_S = x_S)}\right| \le \frac{\sigma}{\delta^{|S|}} + \frac{\sigma}{\delta^{|S|} - \sigma}
\end{align*}
Finally, if $\sigma < \epsilon K^{-\ell} \frac{\delta^{\ell}}{5}$
then because $|S| \le \ell$ and $\sigma < \delta^{\ell}/5 < \delta^{\ell}/2$
\begin{align*}
K^{|S|} \sigma + \sum_{x_S} \left(\sigma + \frac{\sigma}{\delta^{|S|}} + \frac{\sigma}{\delta^{|S|} - \sigma}\right) 
&= K^{|S|} \sigma \left(2 + \frac{1}{\delta^{|S|}} + \frac{1}{\delta^{|S|} - \sigma}\right) \\
&< K^{|S|} \sigma \left(\frac{2}{\delta^{|S|}} + \frac{1}{\delta^{|S|}} + \frac{2}{\delta^{|S|}}\right)
< \epsilon
\end{align*}
This completes the proof.
\end{proof}

\end{document}